\def\year{2020}\relax
\documentclass[letterpaper]{article} 
\usepackage{aaai20}  
\usepackage{times}  
\usepackage{helvet} 
\usepackage{courier}  
\usepackage[hyphens]{url}  
\usepackage{graphicx} 
\usepackage{graphicx}  
\usepackage{url}

\usepackage{amsmath}
\usepackage{amssymb}
\usepackage{amsthm}
\usepackage[ruled]{algorithm2e}

\DeclareMathOperator*{\argmin}{arg\,min}

\newtheorem{lem}{Lemma}

\theoremstyle{definition}
\newtheorem{defn}{Definition}

\title{Fast and Efficient Boolean Matrix Factorization by Geometric Segmentation }
\author{
\Large \textbf{Changlin Wan\textsuperscript{\rm 1,2}, Wennan Chang\textsuperscript{\rm 1,2}, Tong Zhao\textsuperscript{\rm 3}, Mengya Li\textsuperscript{\rm 2}, Sha Cao\textsuperscript{\rm 2,*}, Chi Zhang\textsuperscript{\rm 2,*}}\\ 
\textsuperscript{\rm 1}Purdue University,\textsuperscript{\rm 2}Indiana University, \textsuperscript{\rm 3}Amazon\\ 
\{wan82,chang534\}@purdue.edu, zhaoton@amazon.com, imyli1024@gmail.com, \{shacao,czhang87\}@iu.edu
}
 
\begin{document}

\maketitle

\begin{abstract}
Boolean matrix has been used to represent digital information in many fields, including bank transaction, crime records, natural language processing, protein-protein interaction, etc.
Boolean matrix factorization (BMF) aims to decompose a boolean matrix via the product of two low-ranked boolean matrices, benefiting a number of applications on boolean matrices, e.g., data denoising, clustering, dimension reduction and community detection. 
Inspired by binary matrix permutation theories and geometric segmentation, in this work, we developed a fast and scalable BMF approach, called \textbf{MEBF} (\textbf{M}edian \textbf{E}xpansion for \textbf{B}oolean \textbf{F}actorization). MEBF adopted a heuristic approach to locate binary patterns presented as submatrices that are dense in 1’s. In each iteration, MEBF permutates the rows and columns such that the permutated matrix is approximately \textbf{U}pper \textbf{T}riangular-\textbf{L}ike (\textbf{UTL}) with so-called \textbf{S}imultaneous \textbf{C}onsecutive-ones \textbf{P}roperty (\textbf{SC1P}). The largest submatrix dense in 1 would lie on the upper triangular area of the permutated matrix, and its location was determined based on a geometric segmentation of a triangular. We compared MEBF with state-of-the-art BMF baselines on data scenarios with different density and noise levels. Through comprehensive experiments, MEBF demonstrated superior performances in lower reconstruction error, and higher computational efficiency, as well as more accurate density pattern mining than state-of-the-art methods such as ASSO, PANDA and Message Passing. We also presented the application of MEBF on non-binary data sets, and revealed its further potential in knowledge retrieving and data denoising on general matrix factorization problems. 
\end{abstract}

\section{Introduction}
Binary data gains more and more attention during the transformation of modern living \cite{kocayusufoglu2018summarizing,balasubramaniam2018people}. It consists of a large domain of our everyday life, where the 1s or 0s in a binary matrix can physically mean  whether or not an event of online shopping transaction, web browsing, medical record, journal submission, etc, has occurred or not. The scale of these datasets has increased exponentially over the years. Mining the patterns within binary data as well as adapting to the drastic increase of dimensionality is of prominent interests for nowadays data science research. Recent study also showed that some continuous data could benefit from binary pattern mining. For instance, the binarization of continuous single cell gene expression data to its on and off state, can better reflect the coordination patterns of genes in regulatory networks \cite{larsson2019genomic}. However, owing to its two value characteristics, the rank of a binary matrix under normal linear algebra can be very high due to certain spike rows or columns. This makes it infeasible to apply established methods such as SVD and PCA for BMF \cite{wall2003singular}. \\
Boolean matrix factorization (BMF) has been developed particularly for binary pattern mining, and it factorizes a binary matrix into approximately the product of two low rank binary matrices following Boolean algebra, as shown in Figure 1. The decomposition of a binary matrix into low rank binary patterns is equivalent to locating submatrices that are dense in 1. Analyzing binary matrix with BMF shows its unique power. In the most optimal case, it significantly reduces the rank of the original matrix calculated in normal linear algebra to its log scale \cite{monson1995survey}. Since the binary patterns are usually embedded within noisy and randomly arranged binary matrix, BMF is known to be an NP-hard problem \cite{miettinen2008discrete}. 

\begin{figure*}[t]
\centering
\includegraphics[width=0.9\textwidth]{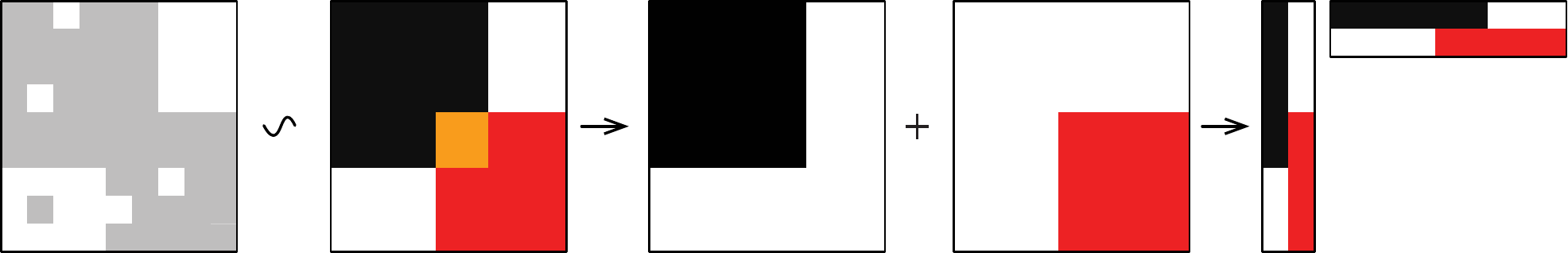}
\caption{BMF, the addition of rank 1 binary matrices}
\label{fig1}
\end{figure*}

\section{Background}
\subsection{Related work}
BMF was first introduced as a set basis problem in 1975 \cite{stockmeyer1975set}. This area has received wide attention after a series of work by Mittenin et al \cite{miettinen2008discrete,miettinen2014mdl4bmf,karaev2015getting}. Among them, the ASSO algorithm performs factorization by retrieving binary bases from row-wise correlation matrix in a heuristic manner \cite{miettinen2008discrete}. Despite its popularity, the high computational cost of ASSO makes it impracticable when dealing with large scale data. Recently,  an algorithm called Nassua was developed by the same group \cite{karaev2015getting}. Nassua optimizes the initialization of the matrix factorization by locating dense seeds hidden within the matrix, and with improved performance comparing to ASSO. However, optimal parameter selection remains a challenge for Nassua. A second series of work called PANDA was  developed by Claudio et al \cite{lucchese2010mining,lucchese2013unifying}. PANDA aims to find the most significant patterns in the current binary matrix by discovering core patterns iteratively \cite{lucchese2010mining}. After each iteration, PANDA only retains a residual matrix with all the non-zero values covered by identified patterns removed. Later, PANDA+ was recently developed to reduce the noise level in core pattern detection and extension \cite{lucchese2013unifying}. These two methods also suffer from inhibitory computational cost, as they need to recalculate a global loss function at each iteration. More algorithms and applications of BMF have been proposed in recent years. FastStep~\cite{araujo2016faststep} relaxed BMF constraints to non-negativity  by integrating non-negative matrix factorization (NMF) and Boolean thresholding. But interpreting derived non-negative bases could also be challenging. With prior network information, Kocayusufoglu et al  decomposes binary matrix in a stepwise fashion with bases that are sampled from given network space \cite{kocayusufoglu2018summarizing}. Bayesian probability mapping has also been applied in this field . Ravanbakhsh et al proposed a probability graph model called “factor-graph” to characterize the embedded patterns, and developed a message passing approach,  called MP \cite{ravanbakhsh2016boolean}. On the other hand, Ormachine, proposed by Rukat et al, provided a probabilistic generative model for BMF \cite{rukat2017bayesian}. Similarly, these Bayesian approaches suffer from  low computational efficiency. In addition, Bayesian model fitting could be highly sensitive to noisy data.

\subsection{Notations}
A matrix is denoted by a uppercase character with a super script \(n\times m\) indicating its dimension, such as  \(X^{n\times m}\), and with subscript \(X_{i,:}\), \(X_{:,j}\), \(X_{ij}\) indicating \(i\)th row, \(j\)th column, or the \((i,j)\)th element, respectively. A vector is denoted as a bold lowercase character, such as \(\textbf{a}\), and its subscript \(\textbf{a}_i\) indicates the \(i\)th element. A scalar value is represented by a lowercase character, such as \(a\), and \([a]\) as its integer part. \(|X|\) and \(|\textbf{x}|\) represents the $\ell_1$ norm of a matrix and a vector. Under the Boolean algebra, the basic operations include \(\wedge(AND, 1\wedge1=1, 1\wedge0=0, 0\wedge0=0)\), \(\vee(OR, 1\vee1=1, 0\vee1=1, 0\vee0=0)\), \(\neg(NOT, \neg1=0, \neg0=1)\). Denote the Boolean element-wise  sum, subtraction and product as \(A\oplus B=A\vee B\), \(A\ominus B=(A\wedge\neg B)\vee(\neg A\wedge B)\) and \(A\circledast B=A\wedge B\), and the Boolean matrix product of two Boolean matrices as \(X^{n\times m}=A^{n\times k}\otimes B^{k\times m}\), where \(X_{ij}=\vee^k_{l=1}A_{il}\wedge B_{lj}\).
\subsection{Problem statement}
Given a binary matrix \(X \in \{ 0,1 \}^{n \times m}\)  and a criteria parameter \(\tau\), the BMF problem is defined as identifying two binary matrices \(A^{*}\) and \(B^{*}\), called pattern matrices, that minimize the cost function  \(\gamma (A,B;X)\) under criteria \(\tau\), i.e., \((A^*,B^*)=\operatorname*{argmin}_{A,B}(\gamma(A,B;X)|\tau)\). Here the criteria \(\tau\) could vary with different problem assumptions. The criteria used in the current study is to identify \(A^*\) and \(B^*\) with at most \(k\) patterns, i.e., \(A\in\{0,1\}^{n\times k}\), \(B\in\{0,1\}^{k\times m}\), and the cost function is \(\gamma(A,B;X)=|X\ominus(A\otimes B)|\). We call the $l$th column of matrix $A$ and $l$th row of matrix $B$ as the $l$th binary pattern, or the $l$th basis, $l=1,...,k$.

\begin{figure}[t]
\centering
\includegraphics[width=0.9\columnwidth]{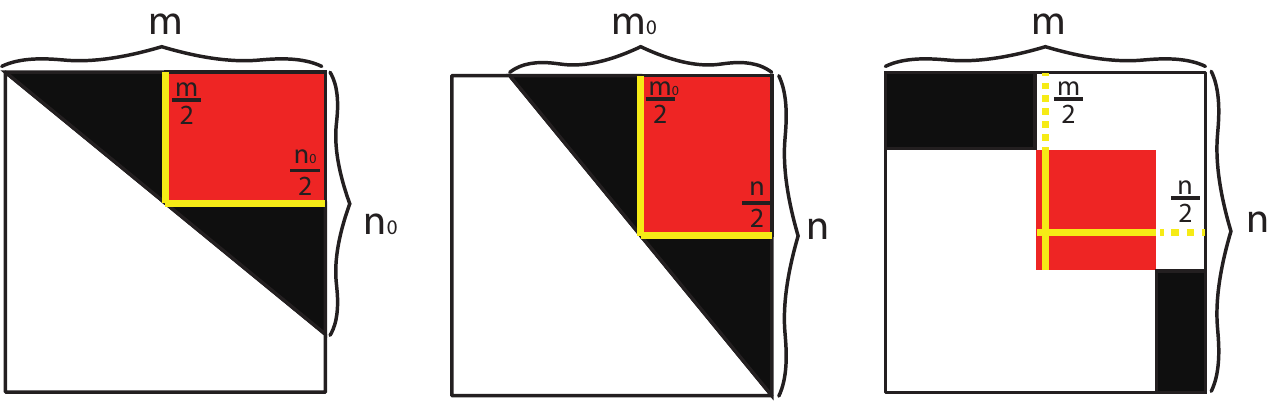}
\caption{Three simplified scenarios for UTL matrices with direct SC1P.}
\label{fig2}
\end{figure}

\begin{figure*}[t]
\centering
\includegraphics[width=\linewidth]{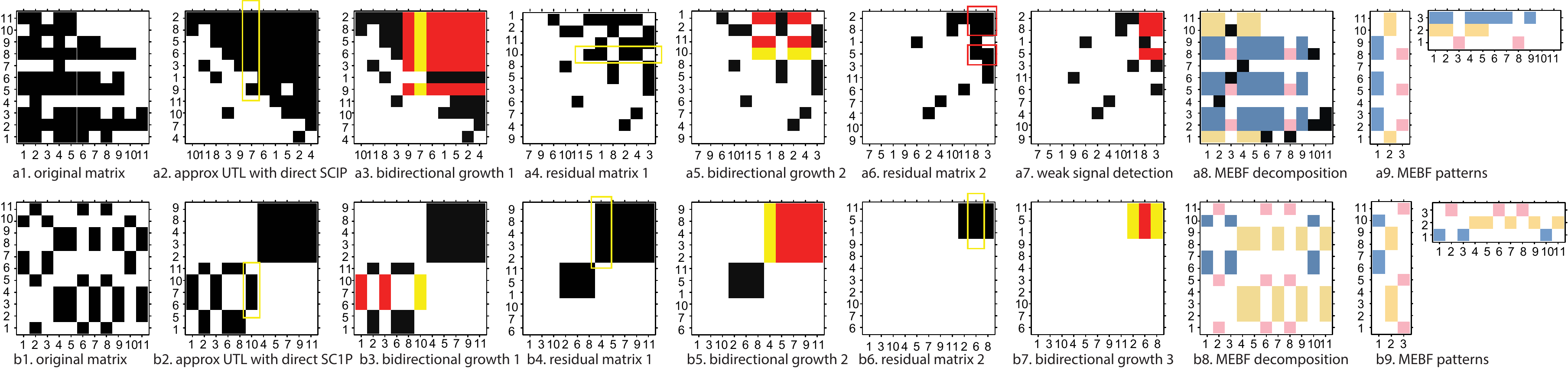}
\caption{A schematic overview of the MEBF pipeline for three data scenarios where the matrix is roughly UTL with SC1P. }
\label{fig3}
\end{figure*}

\section{MEBF Algorithm Framework}
\subsection{Motivation of MEBF}
BMF is equivalent to decomposing the matrix into the sum of multiple rank 1 binary matrices, each of which is also referred as a  pattern or basis in the BMF literature \cite{lucchese2010mining}. 

\begin{lem}[\textbf{Submatrix detection}]
Let \(A^*\),\(B^*\) be the solution to \(\argmin_{A\in\{0,1\}^{n\times k},B\in\{0,1\}^{k\times m}}|X\ominus(A\otimes B)|\), then the k patterns identified in \(A^*\), \(B^*\) correspond to \(k\) submatrices in \(X\) that are dense in 1's. In other words, finding \(A^*\), \(B^*\) is equivalent to identify submatrices \(X_{I_{l},J_{l}}, \, I_{l}\subset \{1,\dots,n\};J_{l}\subset \{1,\dots,m\}, \,l=1,...,k, \, s.t. |X_{I_{l},J_{l}}|\geq t_0(|I_{l}|*|J_{l}|)\). Here \(|I_{l}|\) is the cardinality of the index set \(I_{l}\), \(t_0\) is a positive number between 0 and 1 that controls the noise level of \(X_{I_{l},J_{l}}\). 
\end{lem}

\begin{proof}
\(\forall l\), it suffices to let \(I_{l}\) be the indices of the \(l\)th column of \(A^*\) , such that \(A^*_{:,l}=1\); and let \(J_l\) be the  indices of the \(l\)th row of \(B^*\)  such that \(B^*_{l,:}=1\).
\end{proof}

Motivated by Lemma 1, instead of looking for patterns directly, we turn to identify large submatrices in $X$ that are enriched by 1, such that each  submatrix would correspond to one binary pattern or basis.

\begin{defn}[\textbf{Direct consecutive-ones property, direct C1P}]
A binary matrix \(X\) has direct C1P if for each of its row vector, all 1's occur at consecutive indices.
\end{defn}

\begin{defn}[\textbf{Simultaneous consecutive-ones property, SC1P}]
A binary matrix \(X\)  has direct SC1P, if both \(X\)  and \(X^T\)  have direct C1P; and a binary matrix \(X\) has SC1P, if there exists a permutation of the rows and columns such that the permutated matrix has direct SC1P.
\end{defn}

\begin{defn}[\textbf{Upper Triangular-Like matrix, UTL}]
A binary matrix \(X^{m\times n}\) is called an \textbf{U}pper \textbf{T}riangular-\textbf{L}ike (\textbf{UTL}) matrix, if 1) \(\sum_{i=1}^m X_{i1} \leq \sum_{i=1}^m X_{i2}\leq\cdots\leq\sum_{i=1}^m X_{in}\); 2) \(\sum_{j=1}^n X_{1j} \geq \sum_{j=1}^n X_{2j}\geq\cdots\geq\sum_{j=2}^n X_{mj}\). In other words, the matrix has non-increasing row sums from top down, and non-decreasing column sums from left to right. 
\end{defn}

\begin{lem}[\textbf{UTL matrix with direct SC1P}]
Assume \(X\) has no all-zero rows or columns. If \(X\) is an UTL matrix and has direct SC1P, then an all 1 submatrix of the largest area in \(X\) is seeded where one of its column lies in the medium column of the matrix, or one of its row lies in the medium row of the matrix, as shown in Figure 2.
\end{lem}

Figure 2 presented three simplified scenarios of UTL matrix that has direct SC1P. In (a), (b), the 1’s are organized in triangular shape, where certain rows in (a) and certain columns in (b) are all zero, and in (c), the 1’s are shaped in block diagonal. After removing all-zero rows and columns, the upper triangular area of the shuffled matrix is dense in 1. It is easy to show that a rectangular with the largest area in a triangular is the one defined by the three midpoints of the three sides, together with the vertex of the right angle of the triangular, as colored by red in Figure 2. The width and height of the rectangular equal to half of the two legs of the triangular, i.e. \((\frac{m}{2},\frac{n_0}{2})\), \((\frac{m_0}{2},\frac{n}{2})\), \((\frac{m}{2},\frac{n}{2})\) for the three  scenarios in Figure 2 respectively. According to Lemma 2, this largest rectangular contains at least one row or one column (colored in yellow) of the largest all 1 submatrix in the matrix. Consequently, starting with one row or column, expansions with new rows or columns could be done easily if they show strong similarity to the first row or column. After the expansion concludes, one could determine whether to retain the submatrix expanded row-wise or column-wise, whichever reduces more of the cost function.

It is common that the underlying SC1P pattern may not exist for a binary matrix, and we turn to find the matrix with closest SC1P.
\begin{defn}[\textbf{Closest SC1P}]
Given a binary matrix \(X\)  and a nonnegative weight matrix \(W\), a matrix \(\hat{X}\)  that has SC1P and minimizes the distance \(d_W(X,\hat{X})\) is the closest SC1P matrix of \(X\).
\end{defn}

Based on Lemma 2, we could find all the submatrices in Lemma 1 by first permutating rows and columns of matrix \(X\) to be an UTL matrix with closest direct SC1P, locating the largest submatrix of all 1's, to be our first binary pattern. Then we are left with a residual matrix whose entries covered by existing patterns are set to zero. Repeat the process on the residual matrix until convergence. However, finding matrix of closest SC1P of matrix \(X\) is NP-hard \cite{junttila2011patterns,oswald2009simultaneous}.

\begin{lem}[\textbf{Closest SC1P}]
Given a binary matrix \(X\) and a nonnegative weight matrix \(W\), finding a matrix \(\hat{X}\) that has SC1P and minimizes the distance \(d_W(X,\hat{X})\) is an NP-hard problem.
\end{lem}
The NP-hardness of the closest SC1P problem has been shown in(Oswald and Reinelt 2009, Junttila 2011). Both exact and heuristic algorithms are known for the problem, and it has also been shown if the number of rows or columns is bounded, then solving closest SC1P requires only polynomial time \cite{oswald2003weighted}. In our MEBF algorithm, we attempt to address it by using heuristic methods and approximation algorithms.

\subsection{Overview}

Overall, MEBF adopted a heuristic approach to locate submatrices that are dense in 1's iteratively. Starting with the original matrix as a residual matrix, at each iteration, MEBF permutates the rows and columns of the current residual matrix so that the 1's are gathered on entries of the upper triangular area. This step is to approximate the permutation operation it takes to make a matrix UTL and direct SC1P. Then as illustrated in Figure 2 and Figure 3, the rectangular of the largest area in the upper triangular, and presumably, of the highest frequencies of 1's, will be captured. The pattern corresponding to this submatrix represents a good rank-1 approximation of the current residual matrix. Before the end of each iteration, the residual matrix will be updated by flipping all the 1's located in the identified submatrix in this step to be 0.

Shown in Figure 3a, for an input Boolean matrix (a1), MEBF first rearranges the matrix to obtain an approximate UTL matrix with closest direct SC1P. This was achieved by reordering the rows so that the row norms are non-increasing, and the columns so that the column norms are non-decreasing (a2). Then, MEBF takes either the column or row with medium number of 1's as one basis or pattern (a3). As the name reveals, MEBF then adopts a median expansion step, where the medium column or row would propogate to other columns or rows with a bidirectional growth algorithm until certain stopping criteria is met. Whether to choose the pattern expanded row-wise or column-wise depends on which one minimizes the cost function with regards to the current residual matrix. Before the end of each iteration, MEBF computes a residual matrix by doing a Boolean subtraction of the newly selected rank-1 pattern matrix from the current residual matrix (a4). This process continues until the convergence criteria was met. If the patterns identified by the bidirectional growth step stopped deceasing the cost function before the convergence criteria was met, another step called weak signal detection would be conducted (a6,a7). Figure 3b illustrated a special case, where the permutated matrix is roughly block diagonal (b1), which corresponds to the third scenario in Figure 2. The same procedure as shown in 3a could guarantee the accurate location of all the patterns. The computational complexity of bidirectional growth and weak signal detection algorithms are both O(nm) and the complexity of each iteration of MEBF is O(nm). The main algorithm of MEBF is illustrated below:
\begin{algorithm}
\SetAlgoLined
\textbf{Inputs:} \(X\in\{0,1\}^{n\times m}\), \(t\in(0,1)\),\(\tau\) \par
\textbf{Outputs:} \(A^*\in\{0,1\}^{n\times k}\), \(B^*\in\{0,1\}^{k\times m}\) \par
\(MEBF(X, t, \tau)\):\par
 \(X_{\text{residual}}\leftarrow X\), \(\gamma_0\leftarrow inf\)\par
 \(A^*\leftarrow NULL\), \(B^*\leftarrow NULL\)\par
 \While{\(!\tau\)}{
 \((\textbf{a},\textbf{b}) \leftarrow \text{bidirectional\_growth}(X_{\text{residual}},t)\)\par
 \(A_{\text{tmp}}\leftarrow append(A^*,\textbf{a})\)\par
 \(B_{\text{tmp}}\leftarrow append(B^*,\textbf{b})\)\par
    \uIf{\(\gamma (A_{\text{tmp}},B_{\text{tmp}}; X)>\gamma_0\)}{
   \((\textbf{a},\textbf{b})\leftarrow \text{weak\_signal\_detection}(X_{\text{residual}},t)\)\;
 \(A_{\text{tmp}}\leftarrow append(A^*,\textbf{a})\)\par
 \(B_{\text{tmp}}\leftarrow append(B^*,\textbf{b})\)\par
    \uIf{\(\gamma(A_{\text{tmp}},B_{\text{tmp}};X)>\gamma_0\)}{ break \;}
   }
   \(A^*\leftarrow append(A^*,\textbf{a})\)\par
   \(B^*\leftarrow append(B^*,\textbf{b})\)\par
   \(\gamma_0 \leftarrow \gamma(A^*,B^*; X)\)\par
   \(X_{\text{residual}_{ij}}\leftarrow0 \, \text{when}\, (\textbf{a}\otimes\textbf{b})_{ij}=1\)
 }
 \caption{MEBF}
\end{algorithm}

\subsection{Bidirectional Growth}
For an input binary (residual) matrix \(X\), we first rearrange \(X\) by reordering the rows and columns so that the row norms are non-increasing, and the column norms are non-decreasing. The rearranged \(X\), after removing its all-zero columns and rows, is denoted as \(X'\), the median column and median row of \(X'\) as \(X'_{:,\text{med}}\) and \(X'_{\text{med},:}\). Denote \(X_{:,(\text{med})}\) and \(X_{(\text{med}),:}\) as the column and row in \(X\) corresponding to \(X'_{:,\text{med}}\) and \(X'_{\text{med},:}\). The similarity between \(X_{:,(\text{med})}\)  and columns of \(X\) can be computed as a column wise correlation vector \(\textbf{m}\in(0,1)^m\), where \(\textbf{m}_i=\frac{<X_{:,i},X_{:,(\text{med})}>}{<X_{:,(\text{med})},X_{:,(\text{med})}>}\). Similarly, the similarity between \(X_{(\text{med}),:}\) and rows of \(X\) can be computed as a vector \(\textbf{n}\in(0,1)^n\),\(\textbf{n}_j=\frac{<X_{j,:},X_{(\text{med}),:}>}{<X_{(\text{med}),:},X_{(\text{med}),:}>}\). A pre-specified threshold \(t\in(0,1)\) was further applied, and two vectors \(\textbf{e}\)  and \(\textbf{f}\) indicating the similarity strength of the columns and rows of \(X\) with \(X_{:,(\text{med})}\) and \(X_{(\text{med}),:}\), are obtained, where \(\textbf{e}_j=(\textbf{m}_j>t)\) and \(\textbf{f}_i=(\textbf{n}_i>t)\). Here the binary vectors \textbf{e} and \textbf{f} \textbf{each represent one potential BMF  pattern }. In each iteration, we select the row or column pattern whichever fits the current residual matrix better,  i.e. the column pattern if \(\gamma(X_{:,(\text{med})},\textbf{e};X)<\gamma(\textbf{f},X_{(\text{med}),:};X)\), or the row pattern otherwise. Here, the cost function is defined as \(\gamma(\textbf{a},\textbf{b};X)=|X\ominus (\textbf{a}\otimes \textbf{b})|\). This is equivalent to selecting a pattern that achieves lower overall cost function at the current step. Obviously here, a smaller \(t\) could achieve higher coverage with less number of patterns, while a larger \(t\) enables a more sparse decomposition of the input matrix with greater number of patterns. Patterns found by bidirectional growth does not guarantee a constant decrease of the cost function. In the case the cost function  increases, we adopt a weak signal detection step before stopping the algorithm.

\begin{algorithm}
\SetAlgoLined
\textbf{Inputs:} \(X\in \{0,1\}^{n\times m}\), \(t\in(0,1]\) \\
\textbf{Outputs:} (\textbf{a},\textbf{b}) \\
\(\text{bidirectional\_growth}(X,t):\)\\
\(X'\leftarrow \textbf{UTL operation on }  X \)\\
\(\textbf{d} \leftarrow X_{:,(\text{med})}\), \(\textbf{e}\leftarrow \{(\frac{<X_{:,j},\textbf{d}>}{<\textbf{d},\textbf{d}>} >t),j=1,...,m\}\)\\
\( \textbf{f} \leftarrow X_{(\text{med}),:}\), \(\textbf{g}\leftarrow \{(\frac{<X_{i,:},\textbf{f}>}{<\textbf{f},\textbf{f}>} >t),i=1,...,n\}\)\\
\eIf{\(\gamma(\textbf{d},\textbf{e};X)>\gamma(\textbf{g},\textbf{f};X)\)}{
   \(\textbf{a}\leftarrow\textbf{g}\); \(\textbf{b}\leftarrow\textbf{f}\)\;
   }{
   \(\textbf{a}\leftarrow\textbf{d}\); \(\textbf{b}\leftarrow\textbf{e}\)\;
  }
 \caption{Bidirectional Growth}
\end{algorithm}

\subsection{Weak Signal Detection Algorithm}
The bidirectional growth steps do not guarantee a constant decrease of the cost function, especially when after the "large" patterns have been identified and the "small" patterns are easily confused with noise. To identify  weak patterns from a residual matrix, we came up with a week signal detection algorithm to locate the regions that may still have small but true patterns. Here, from the current residual matrix, we search the two columns with the most number of 1's and form a new column that is the intersection of the two columns; and the two rows with the most number of 1's and form a new row that is the intersection of the two rows. Starting from the new column and new row as a pattern, similar to bidirectional growth, we locate the rows or columns in the residual matrix that have high enough similarity to the pattern, thus expanding a single row or column into a submatrix. The one pattern among the two with the lowest cost function with regards to the residual matrix will be selected. And if addition of the pattern to existing patterns could decrease the cost function with regards to the original matrix, it will be retained. Otherwise, the algorithm will stop.
\begin{algorithm}
\SetAlgoLined
\textbf{Inputs:} \(X\in \{0,1\}^{n\times m}\), \(t\in(0,1]\) \\
\textbf{Outputs:} \((\textbf{a},\textbf{b})\) \\
Weak\_signal\_detection(\(X,t\))\\
 \(X'\leftarrow \textbf{UTL operation on } X\)\\
 \(\textbf{d}^1 \leftarrow X'_{:,m}\wedge X'_{:,m-1} \) \\ 
 \(\textbf{e}^1\leftarrow \{(\frac{<X_{:,j},\textbf{d}^1>}{<\textbf{d}^1,\textbf{d}^1>} >t),j=1,...,m\}\)\\
 \(\textbf{e}^2 \leftarrow X'_{1,:}\wedge X'_{2,:}\) \\
 \(\textbf{d}^2\leftarrow \{(\frac{<X_{i,:},\textbf{e}^2>}{<\textbf{e}^2,\textbf{e}^2>} >t),i=1,...,n\}\)\\

\(l \leftarrow \argmin_{l=1, 2}\gamma(\textbf{d}^l,\textbf{e}^l,X)\)

   \(\textbf{a}\leftarrow\textbf{d}^l\); \(\textbf{b}\leftarrow\textbf{e}^l\)\;

 \caption{Weak Signal Detection}
\end{algorithm}

\begin{figure*}
\centering
\includegraphics[width=\textwidth]{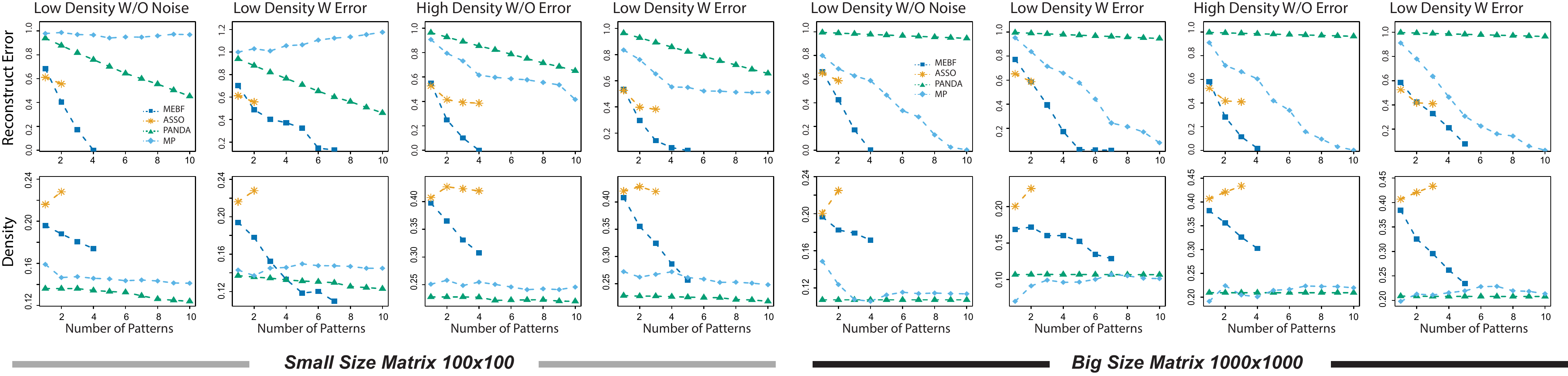}
\caption{Performance comparisons of MEBF, ASSO, PANDA and MP on the accuracy of decomposition. }
\label{fig4}
\end{figure*}

\begin{figure}[t]
\centering
\includegraphics[width=\columnwidth]{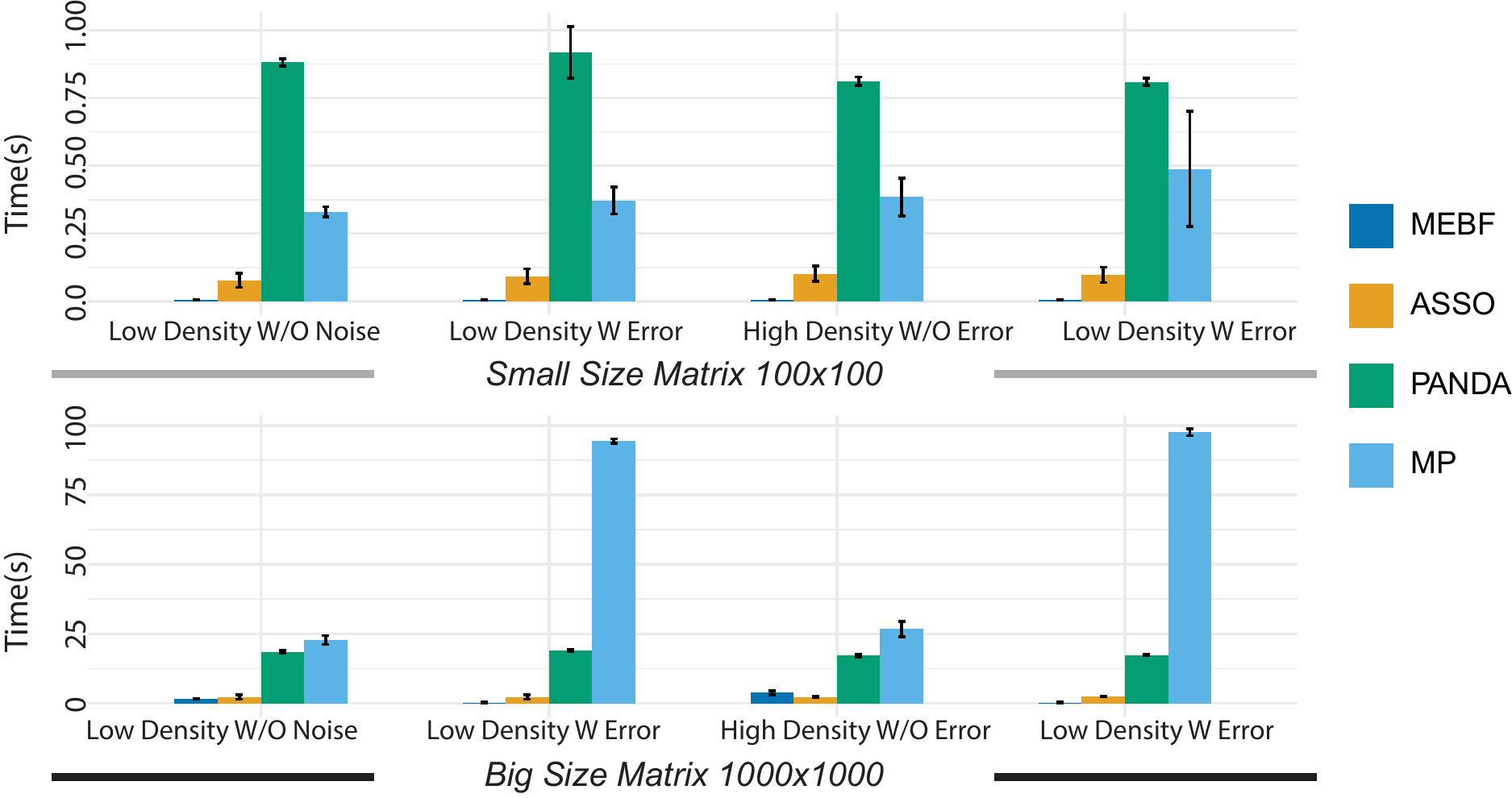}
\caption{Performance comparison of MEBF, ASSO, PANDA and MP on computational cost. }
\label{fig5}
\end{figure}

\section{Experiment}
\subsection{Simulation data}
We first compared MEBF\footnote{The code is available at https://github.com/clwan/MEBF} with three state-of-the-art approaches, ASSO, PANDA and Message Passing (MP), on simulated datasets. \\
A binary matrix \(X^{n\times m}\) is simulated as \[X^{n\times m}=U^{n\times k}\otimes V^{k\times m}+_f E\] where \[U_{ij},V_{ij}\sim Bernoulli(p_0) \quad E_{ij}\sim Bernoulli(p)\] \("+_f"\) is a flipping operation, s.t. \[X_{ij}=\begin{cases}
\vee_{l=1}^k U_{il}\wedge V_{lj},\quad E_{ij}=0\\
\neg \vee_{l=1}^k U_{il}\wedge V_{lj},\quad E_{ij}=1
\end{cases}\]
Here, \(p_0\) controls the density levels of the true patterns, and \(E\) is introduced as noise that could flip the binary values, and the level of noise could be regulated by the parameter \(p\).\\
We simulated two data scales, a small one, $n=m=100$, and a large one $n=m=1000$. For each data scale, the number of patterns \(k\), is set to 5, and we used two density levels, where \(p_0=0.2,0.4\), and two noise levels $p=0,0.01$. 50 simulation was done for each data scale at each scenario.\\
We evaluate the goodness of the algorithms by considering two metrics, namely the reconstruction error and density \cite{belohlavek2018toward,rukat2017bayesian}, as defined below:\\
\[\text{Reconstruction error} := \frac{|(U\otimes V)\ominus(A^*\otimes B^*)|}{|U\otimes V|}\]\\
\[\text{Density} := \frac{|A^{*n\times k}|+|B^{*k\times m}|}{(n+m)\times k}\]\\
Here, \(U\), \(V\) are the ground truth patterns while \(A^*\) and \(B^*\) are the decomposed patterns by each algorithm. The density metric is introduced to evaluate whether the decomposed patterns could reflect the  sparsity/density levels of the true patterns. It is notable that with the same reconstruction error, patterns of lower density, i.e., higher sparsity are more desirable, as it leads to more parsimonious models.\\
In Figure 4 and 5, we show that, compared with ASSO, PANDA and MP, MEBF is the fastest and most robust algorithm. Here, the convergence criteria for the algorithms are set as: (1) 10 patterns were identified; (2) or for MEBF, PANDA and ASSO, they will also stop if a newly identified pattern does not decrease the cost function. \\
As shown in Figure 4, MEBF has the best performance on small and big sized matrices for all the four different scenarios, on 50 simulations each. It achieved the lowest reconstructed error with the least computation time compared with all other algorithms. The convergence rate of MEBF also outperforms PANDA and MP. Though ASSO converges early with the least number of patterns, its reconstruction error is considerably higher than MEBF, especially for high density matrices. In addition, ASSO derived patterns tend to be more dense than the true patterns, while those derived from the other three methods have similar density levels with the true patterns. By increasing the number of patterns, PANDA stably decreased reconstruction error, but it has a considerably slow convergence rate and high computation cost. MP suffered in fitting small size matrices, and in the case of low density matrix with noise, MP derived patterns would not converge.  The standard deviations of reconstruction error and density across 50 simulations is quite low, and was demonstrated by the size of the shapes. The computational cost and its standard deviation for each algorithm is shown as bar plots in Figure 5, and clearly, MEBF has the best computational efficiency among all.

\begin{figure*}[t]
\centering
\includegraphics[width=\textwidth]{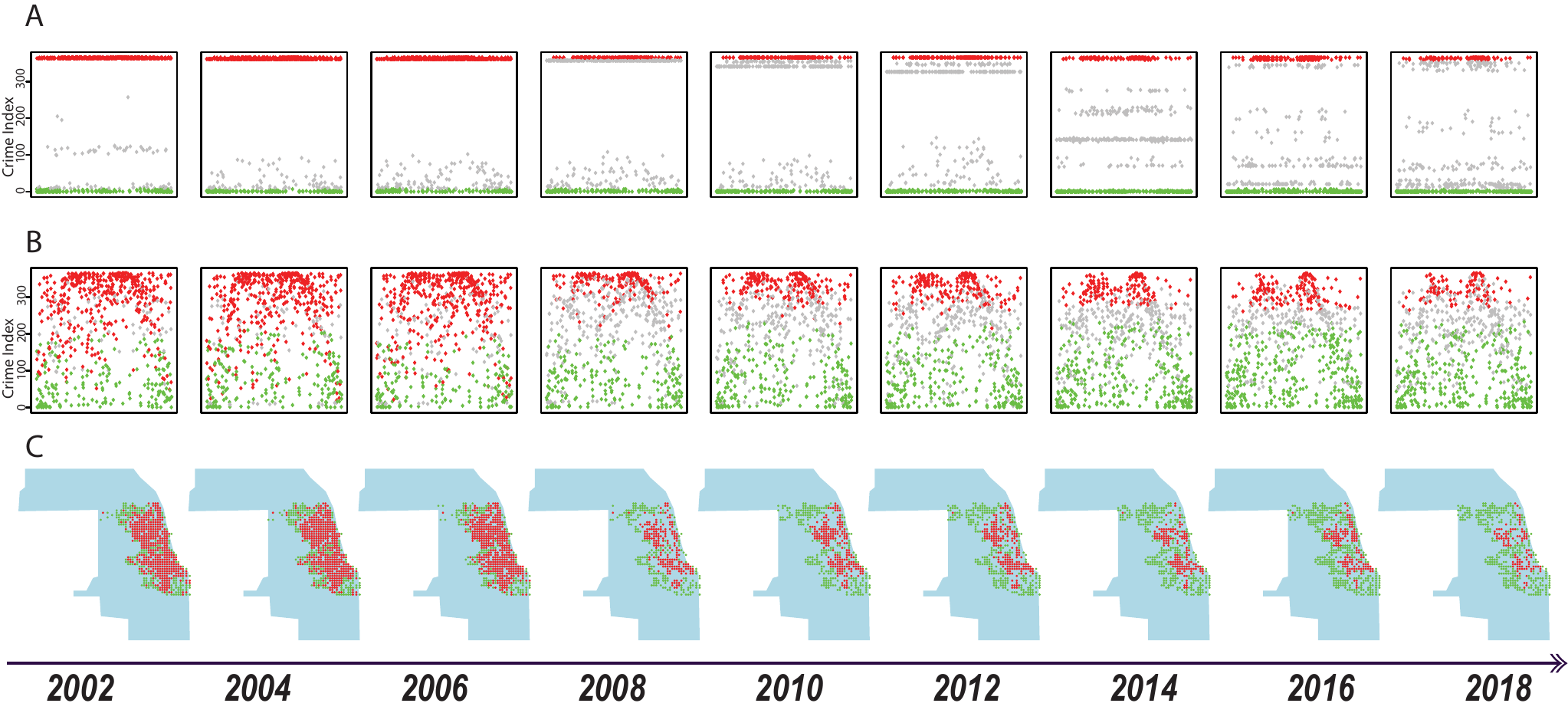}
\caption{MEBF analysis of Chicago crime data over the years.}
\label{fig6}
\end{figure*}

\subsection{Real world data application}
We next focus on the application of MEBF on two real world datasets, and its performance comparison with MP. Both datasets, \textit{Chicago Crime records}\footnote{Chicago crime records downloaded on August 20, 2019 from https://data.cityofchicago.org/Public-Safety} (\(X\in \{0,1\}^{6787\times 752}\)) and \textit{Head and Neck Cancer Single Cell RNA Sequencing data}\footnote{This head and neck sequencing data can be accessed at https://www.ncbi.nlm.nih.gov/geo/query/acc.cgi?acc=GSE103322} (\(X\in \{0,1\}^{344\times 5902}\)), are publicly available. The computational cost of ASSO and PANDA are too inhibitive to be applied to datasets of such a large scale, so they were dropped from the comparisons. Due to a lack of ground-truth of the two low rank construction matrices and the possible high noise level in the real world datasets, it may not be reasonable to examine the reconstruction error with respect to the original matrix. Instead, we focused on two metrics, the density and coverage levels. Density metric was defined as in the simulation data, and coverage rate is defined as \\
\[\text{Coverage rate} :=\frac{|(X\cdot(A^*\otimes B^*))|}{|X|}\]\\
With the same reconstruction error, binary patterns are more desirable to have high sparsity, meaning low density levels, as it makes further interpretation easier and avoids possible overfitting. On the other hand, in many real world data, 0 is more likely to be a false negative occurrence, compared with 1 being a false positive occurrence. In this regard, a higher coverage rate, meaning higher recovery of the 1's, would be a more reasonable metric than lower reconstruction error to the noisy original matrix, as 0's are more likely to be noisy observations than 1's.\\
We compared MEBF and MP for \(k=5\) and \(k=20\), and the density and coverage rate of the derived patterns and time consumption of the two algorithms are presented in Table 1. Overall, as shown in Table 1, for both \(k=5\) and \(k=20\), MEBF outperforms MP in all three measures： higher coverage rate, roughly half the density levels to MP, and  the time consumption of MEBF is approximately 1\% to that of MP. Also noted, with the increased number of  patterns, the coverage rate of MP unexpectedly drops from 0.812 to 0.807 in the crime data, suggesting the low robustness of MP.

Next we discuss in detail the application of BMF on discrete data mining and continuous data mining, and present the findings on the two datasets using MEBF.

\begin{table}[htbp]
\centering
\resizebox{\columnwidth}{!}{\begin{tabular}{l l l l}
\hline\hline
 &  \(\text{Coverage}_\text{(MEBF/MP)}\) & \(\text{Density}_ \text{(MEBF/MP)}\) & \(\text{Time/s}_\text{(MEBF/MP)}\)\\ [1.3ex]
\hline
\(\text{Crime}_{\text{k=5}}\)& 0.835/0.812 & 0.019/0.027 &2.913/333.601\\ [1ex]
\(\text{Crime}_{\text{k=20}}\)& 0.891/0.807 & 0.030/0.066 &10.608/992.011\\ [1ex]
\(\text{Single cell}_{\text{k=5}}\) & 0.496/0.463 & \(2.06\text{e-4}\)/\(2.86\text{e-4}\) & 1.846/137.623\\[1ex]
\(\text{Single cell}_{\text{k=20}}\) & 0.626/0.580 & \(3.34\text{e-4}\)/\(7.22\text{e-4}\) & 5.954/390.217\\[0.6ex]
\hline\hline
\end{tabular}}
\caption{Comparison of MEBF and MP on real world data}\label{tab1}
\end{table}

\textbf{\textit{Discrete data mining}}\\
Chicago is the most populous city in the US Midwest, and it has one of the highest crime rates in the US. It has been well understood that the majority of crimes such as theft and robbery have strong date patterns. For example, crimes were committed for the need to repay regular debt like credit cards, which has a strong date pattern in each month. Here we apply MEBF in analyzing Chicago crime data from 2001 to 2019 to find crime patterns on time and date for different regions. The crime patterns is useful for the allocation of police force, and could also reflect the overall standard of living situation of regions in general.\\
We divided Chicago area into \(\sim 800\) regions of roughly equal sizes. For each of the 19 years, a binary matrix \(X^{n\times m}_d\) for the \(d\)th year is constructed, where \(n\) is the total dates in each year and \(m\) represents the total number of regions, and \(X_{d_{i,j}}=1\) means that crime was reported at date \(i\) in region \(j\) in year \(d\), \(X_{d_{i,j}}=0\) otherwise. MEBF was then applied on each of the constructed binary matrices with parameters \((t=0.7, k=20)\) and outputs \(A^{n\times k}_d\) and \(B^{k\times m}_d\). The reconstructed binary matrix is accordingly calculated as \(X_{d^*}=A^{n\times k}_d\otimes B^{k\times m}_d\). A crime index was defined as the total days with crime committed for regions \(j\) in year \(d\), \(C_{d_j}=\sum_{i=1}^m X_{d_{i,j}}\) .

Figure 6 shows the crime patterns over time, and only even years were shown due to space limit. In 6A, from year 2002-2018, the crime index calculated from the reconstructed matrix, namely, \(C_{d^*_j}=\sum_{i=1}^m X_{d^*_{i,j}}\)  was shown on the \(y\)-axis for all the regions on \(x\)-axis, In 6A, points colored in red indicate those regions with crime index equal to total dates of the year, i.e., 365 or 366, meaning these regions are heavily plagued with crimes, such that there is no day without crime committed. Points colored in green shows vice versa, indicating those regions with no crimes committed over the year. Points are otherwise colored in gray. 6B shows the crime index on the original matrix, and clearly, the green and red regions are distinctly separated, i.e. green on the bottom with low crime index, and red on the top with high crime index. This shows the consistency of the crime patterns between the reconstructed and original crime data, and thus, validate the effectiveness of MEBF pattern mining. Notably, the dramatic decrease in crime index starting from 2008 as shown in Figure 6A and 6B correlates with the reported crime decrease in Chicago area since 2008. Figure 6C shows the crime trend over the years on the map of Chicago. Clearly, from 2008 to now, there is a gradual increase in the green regions, and decrease in the red regions, indicating an overall good transformation for Chicago. This result also indicate that more police force could be allocated in between red and green regions when available. 

\textbf{\textit{Continuous data denoising}}\\
\begin{figure}[t]
\centering
\includegraphics[width=\columnwidth]{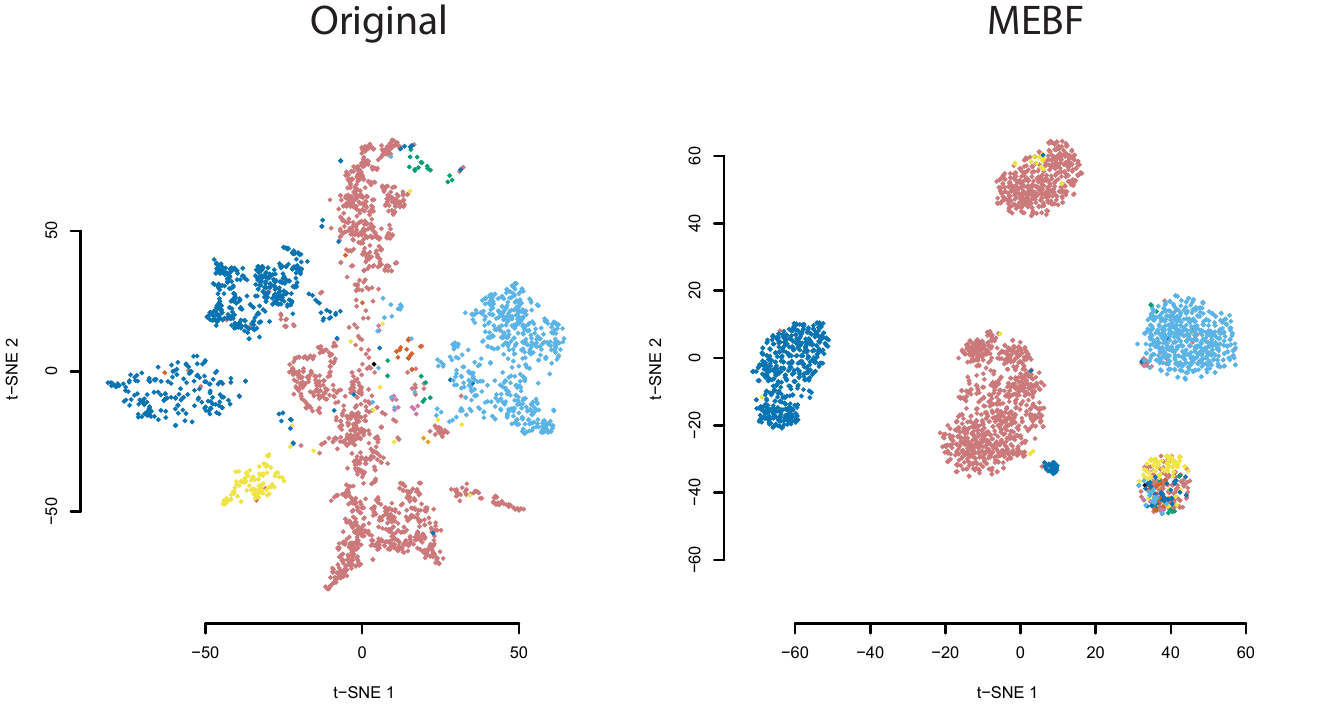}
\caption{Visualization of single cell clustering before and after MEBF.}
\label{fig7}
\end{figure}
Binary matrix factorization could also be helpful in continuous matrix factorization, as the Boolean rank of a matrix could be much smaller than the  matrix rank under linear algebra. Recently, clustering of single cells using single-cell RNA sequencing data has gained extensive utilities in many fields, wherein the biggest challenge is that the high dimensional gene features, mostly noise features, makes the distance measure of single cells in clustering algorithm highly unreliable. Here we aim to use MEBF to denoise the continuous matrix while clustering.\\
We collected a single cell RNA sequencing (scRNAseq) data \cite{puram2017single}, that measured more than 300 gene expression features for over 5,000 single cells, i.e., \(X^{5000\times 300}\). We first binarize original matrix \(X\) into \(X^*\), s.t. \(X_{ij}^*=1\) where \(X_{ij}>0\),and \(X_{ij}^*=0\) otherwise. Then, applying MEBF on \(X^*\) with parameters \((t=0.6,k=5)\) outputs \(A^{n\times k}\), \(B^{k\times m}\). Let \(X^{**}=A\otimes B\) and \(X_{\text{use}}\) be the inner product of \(X^*\) and \(X^{**}\), namely, \(X_{\text{use}}=X\circledast X^{**}\). \(X_{\text{use}}\) represents a denoised version of \(X\), by retaining only the entries in \(X\) with true non-zero gene expressions. And this is reconstructed from the hidden patterns extracted by MEBF.

As shown in Figure 7, clustering on the denoised expression matrix, \(X_{\text{use}}\), results in much tighter and well separated clusters (right) than that on the original expression matrix (left), as visualized by \textit{t}-SNE plots shown in Figure 7. \textit{t}-SNE is an non-linear dimensional reduction approach for the visualization of high dimensional data \cite{maaten2008visualizing}. It is worth noting that, in generating Figure 7, Boolean rank of 5 was chosen for the factorization, indicating that the heterogeneity among cell types with such a high dimensional feature space could be well captured by matrices of Boolean rank equal to 5. Interestingly, we could see a small portion of fibroblast cell (dark blue) lies much closer to cancer cells (red) than to the majority of the fibroblast population, which could biologically indicate a strong cancer-fibroblast interaction in cancer micro-environment. Unfortunately, such interaction is not easily visible in the clustering plot using original noisy matrix.

\section{Conclusion}
We sought to develop a fast and efficient algorithm for boolean matrix factorization, and adopted a heuristic approach to locate submatrices that are dense in 1's iteratively, where each such submatrix corresponds to one binary pattern in BMF. The submatrix identification was inspired by binary matrix permutation theory and geometric segmentation. Approximately, we permutate rows and columns of the input matrix so that the 1's could be "driven" to the upper triangular of the matrix as much as possible, and a dense submatrix could be more easily geometrically located. Compared with three state of the art methods, ASSO, PANDA and MP, MEBF achieved lower reconstruction error, better density and much higher computational efficiency on simulation data of an array of situations. Additionally, we demonstrated the use of MEBF on discrete data pattern mining and continuous data denoising, where in both case, MEBF generated consistent and robust findings.

\section{Acknowledgment}
This work was supported by R01 award \#1R01GM131399-01, NSF IIS (N0.1850360), Showalter Young Investigator Award from Indiana CTSI and Indiana University Grand Challenge Precision Health Initiative.

\bibliographystyle{aaai}
\bibliography{MEBF.bib}

\end{document}